\documentclass[letterpaper]{article} 
\usepackage{aaai2026}  
\usepackage{times}  
\usepackage{helvet}  
\usepackage{courier}  
\usepackage[hyphens]{url}  
\usepackage{graphicx} 
\usepackage{natbib}  
\usepackage{caption} 
\usepackage{algorithm}
\usepackage{algorithmic}

\usepackage{multirow}
\usepackage{booktabs}
\usepackage{amsmath}
\usepackage{amssymb}
\usepackage{amsfonts}
\usepackage{bm}  
\usepackage{amsthm}
\newtheorem{theorem}{Theorem}
\newtheorem*{theorem*}{Theorem}
\newtheorem{lemma}{Lemma}
\newtheorem*{lemma*}{Lemma}
\newtheorem{definition}{Definition}
\newtheorem{proposition}{Proposition}

\usepackage{newfloat}
\usepackage{listings}
\DeclareCaptionStyle{ruled}{labelfont=normalfont,labelsep=colon,strut=off} 
\floatstyle{ruled}
\newfloat{listing}{tb}{lst}{}
\floatname{listing}{Listing}
\title{Adaptive Hyperbolic Kernels: Modulated Embedding \\in de Branges-Rovnyak Spaces}
\author {
    Leping Si\textsuperscript{\rm 1,\rm 2}\equalcontrib,
    Meimei Yang\textsuperscript{\rm 1,\rm 2}\equalcontrib,
    Hui Xue\textsuperscript{\rm 1,\rm 2}\thanks{Corresponding Author},
    Shipeng Zhu\textsuperscript{\rm 1,\rm 2},
    Pengfei Fang\textsuperscript{\rm 1,\rm 2}
}
\affiliations {
    \textsuperscript{\rm 1}School of Computer Science and Engineering Southeast University, Nanjing 210096, China\\
    \textsuperscript{\rm 2}Key Laboratory of New Generation Artificial Intelligence Technology and Its Interdisciplinary Applications (Southeast University), Ministry of Education, China\\
    230248997@seu.edu.cn, meimeiyang01@gmail.com, hxue@seu.edu.cn, shipengzhu@seu.edu.cn, fangpengfei@seu.edu.cn
}

\usepackage{bibentry}

\begin{document}

\maketitle

\begin{abstract}
Hierarchical data pervades diverse machine learning applications, including natural language processing, computer vision, and social network analysis. Hyperbolic space, characterized by its negative curvature, has demonstrated strong potential in such tasks due to its capacity to embed hierarchical structures with minimal distortion. Previous evidence indicates that the hyperbolic representation capacity can be further enhanced through kernel methods. 
However, existing hyperbolic kernels still suffer from mild geometric distortion or lack adaptability. This paper addresses these issues by introducing a curvature-aware de Branges–Rovnyak space, a reproducing kernel Hilbert space (RKHS)  that is isometric to a Poincar\'e ball. 
We design an adjustable multiplier to select the appropriate RKHS corresponding to the hyperbolic space with any curvature adaptively. Building on this foundation, we further construct a family of adaptive hyperbolic kernels, including the novel adaptive hyperbolic radial kernel, whose learnable parameters modulate hyperbolic features in a task-aware manner. Extensive experiments on visual and language benchmarks demonstrate that our proposed kernels outperform existing hyperbolic kernels in modeling hierarchical dependencies.
\end{abstract}

\begin{links}
    \link{Code}{https://github.com/daslp/De-Branges-Rovnyak-Kernel.git}
    \link{Extended version}{https://doi.org/10.48550/arXiv.2511.09921}
\end{links}

\section{Introduction}

Hierarchical structures are prevalent in real-world data across various machine learning domains, such as natural language processing (NLP), computer vision (CV), and social network analysis \cite{Mettes2024, DBLP:journals/pami/PengVMSZ22}. Hyperbolic space, owing to its exponential expansion property, provides a more suitable geometric framework for representing such hierarchical data than the commonly used Euclidean space. As illustrated in Figure~\ref{fig1}, hyperbolic space allows tree-like data with hierarchical structure to spread out without distortion, whereas embedding tree-like data in Euclidean space often results in crowding and overlap.

To better capture complex hierarchical structures, hyperbolic geometry has been introduced into machine learning as an alternative to Euclidean geometry. Nickel et al.\  pioneered hyperbolic embeddings by optimizing them on Riemannian manifolds, demonstrating significant gains in textual data~\cite{DBLP:conf/nips/NickelK17}. Since then, hyperbolic methods have been applied to a wide range of machine learning tasks, such as image classification and graph node prediction, leveraging the strong representational capacity of hyperbolic space~\cite{DBLP:conf/nips/GaneaBH18,DBLP:conf/iclr/ShimizuMH21, DBLP:journals/corr/abs-1805-09786, DBLP:conf/nips/ChamiYRL19,DBLP:conf/cvpr/KhrulkovMUOL20,DBLP:conf/iclr/BdeirSL24,DBLP:journals/ijcv/FangHLP23}.

\begin{figure}[t]
\centering
\includegraphics[width=0.95\columnwidth]{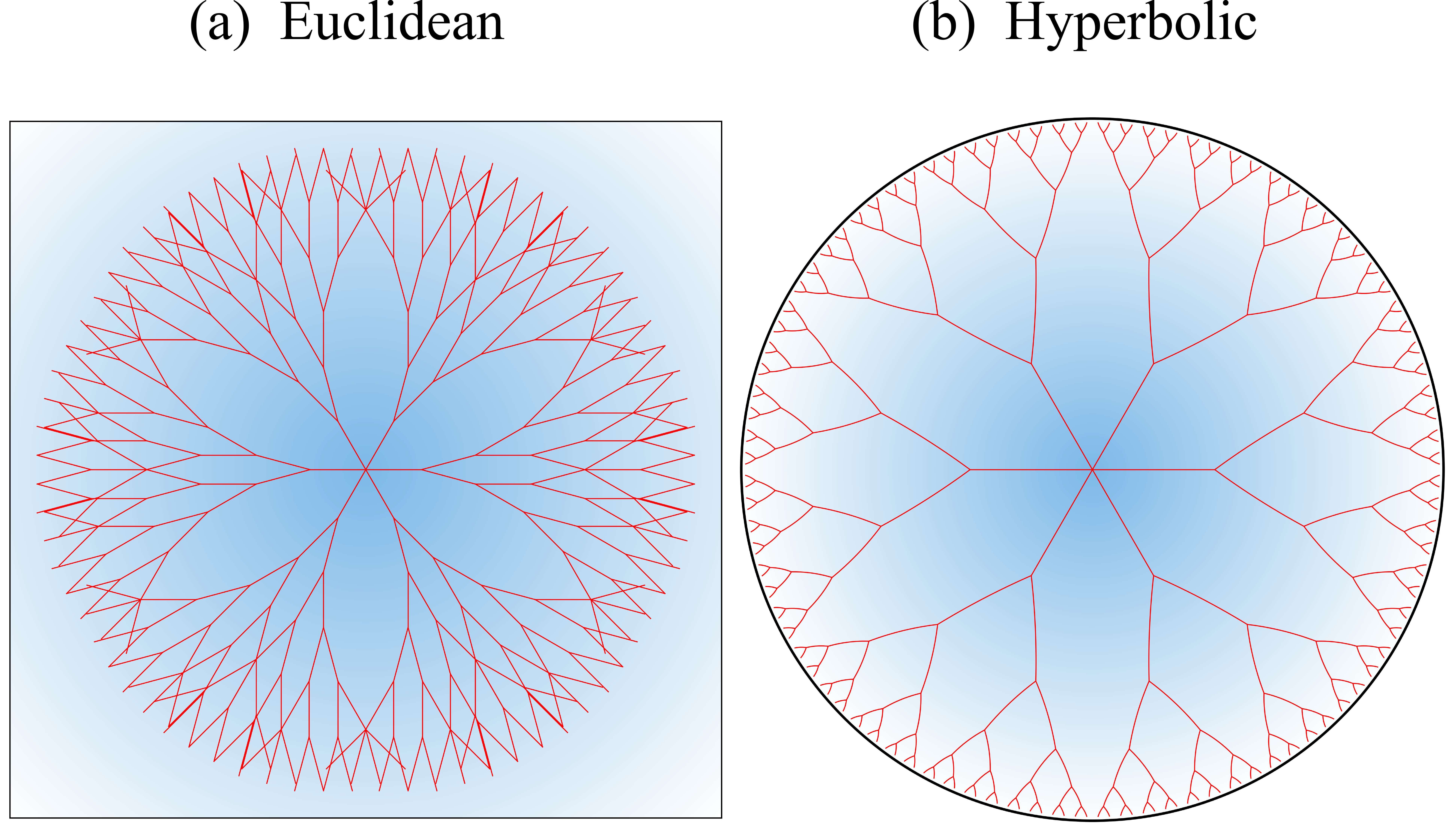}
\caption{Embedding of the same tree characterized by identical branching angles and branch lengths (with hierarchical node structure) in Euclidean and hyperbolic spaces. The left figure shows the embedding in Euclidean space, where some branches overlap. The right figure illustrates the embedding in hyperbolic space, where the exponential expansion property enables distortion-free embedding of the tree.
}
\label{fig1}
\end{figure}

Recent studies have shown that integrating kernel methods with hyperbolic embeddings can further enhance their representational capacity. 
Cho et al.\ introduced the hyperbolic kernel SVM \cite{DBLP:conf/aistats/ChoD0B19}, which constructs hyperbolic kernels based on an isometric mapping between hyperbolic models.   
However, these kernels are not positive definite (pd) and assume a fixed curvature value, limiting the stability and flexibility. Fang et al.\ \cite{DBLP:conf/iccv/FangHP21,DBLP:journals/ijcv/FangHLP23}  addressed this limitation by proposing a family of pd hyperbolic kernels with curvature parameters based on the tangent space of Poincar\'e ball. But the first-order approximation introduces geometric distortions. Yang et al.\ \cite{DBLP:conf/ijcai/YangFX23} reduced the distortion by constructing curvature-aware kernels that map data from hyperbolic space to reproducing kernel Hilbert spaces (RKHS) isometrically. Despite their progress, most existing methods still suffer from fixed functional forms, which can lead to over-representation or structural underfitting and then reduce adaptability.  Consequently, a key challenge is to design hyperbolic kernels that not only preserve the underlying geometry but also adapt flexibly to task-driven requirements.

 To this end, we first construct a curvature-aware de Branges–Rovnyak space~\cite{ball2014branges}, an RKHS that is isometric to Poincar\'e ball, thereby preserving hyperbolic geometry with minimal distortion. Moreover we introduce an adjustable multiplier that allows us to select the appropriate de Branges–Rovnyak space adaptively. Leveraging this mechanism, we further construct a family of adaptive hyperbolic kernels that includes the hyperbolic counterparts of the linear, polynomial, RBF, and Laplacian kernels, along with a distinctive member, the adaptive hyperbolic radial kernel (AHRad).  Thanks to the learnable parameters, AHRad can adaptively enhance or suppress hyperbolic features, enabling task-aware modulation of representations. Overall, the framework flexibly controls hyperbolic representations and encodes hierarchical information with minimal distortion across diverse applications.

Our main contributions are summarized as follows:
\begin{itemize}
\item We construct the curvature-aware de Branges–Rovnyak kernel that realizes an isometric mapping from hyperbolic space with arbitrary curvature to the de Branges–Rovnyak space, thereby providing a rigorous bridge between hyperbolic geometry and RKHS.

\item  We introduce an adjustable multiplier within the curvature-aware de Branges–Rovnyak space, yielding a new formulation of hyperbolic kernels that can adaptively select the RKHS best matched to a hyperbolic space with any given curvature.

\item We develop a series of adaptive hyperbolic kernels,  including hyperbolic linear, polynomial, RBF, and Laplacian kernels, as well as an adaptive hyperbolic radial kernel (AHRad), to enhance both representational power and flexibility.

\item We validate our approach with extensive experiments on diverse tasks, including zero-shot and few-shot image recognition, as well as semantic textual similarity in NLP (e.g., STS-B), and demonstrate that our adaptive hyperbolic kernels outperform existing methods.

\end{itemize}

\section{Related Work}
\subsection{Hyperbolic Kernel Learning}
Hyperbolic space has shown great potential in modeling hierarchical structures due to its exponential growth property, and kernel methods can further enhance the representational capability of hyperbolic embeddings. Early foundational work introduced hyperbolic embeddings for textual data\cite{DBLP:conf/nips/NickelK17} and image \cite{DBLP:conf/cvpr/KhrulkovMUOL20}, laying a solid foundation for subsequent hyperbolic kernel learning paradigms.

In 2019, Cho et al. \cite{DBLP:conf/aistats/ChoD0B19} first proposed a hyperbolic support vector machine (SVM) equipped with the hyperbolic polynomial kernel, an indefinite kernel defined in the hyperboloid model. Their work demonstrated that the optimization formulation, equipped with the Minkowski inner product in the hyperboloid model, closely resembles Euclidean SVMs. This approach achieved improved classification performance on graph-structured and language data. To address the instability issues caused by indefinite kernels and the limited flexibility arising from a fixed curvature value, Fang et al. proposed various pd hyperbolic kernel functions \cite{DBLP:conf/iccv/FangHP21, DBLP:journals/ijcv/FangHLP23}. These methods first project hyperbolic data onto the tangent space and then construct hyperbolic kernels by integrating the mapped features into Euclidean kernels. Although these kernels demonstrated strong performance in various computer vision tasks, they still suffer from distortions due to their first-order approximation of hyperbolic geometry. To mitigate such distortions, Yang et al. \cite{DBLP:conf/ijcai/YangFX23} proposed a novel approach inspired by the isometry between hyperbolic spaces and certain RKHSs~\cite{arcozzi2007diameter}. They designed a series of hyperbolic kernels based on this isometry,  enhancing hyperbolic representation power and achieving superior performance in graph learning and computer vision tasks.

\section{Notations and Preliminaries}

\subsection{Notations}

Throughout this paper, we let $\mathbb{R}^n$ denote the $n$-dimensional real vector space, $\mathbb{C}^n$ the $n$-dimensional complex vector space, $\mathbb{B}^n(c)$ an open ball of radius $\frac{1}{\sqrt{c}}$ in $\mathbb{C}^n$, and $\mathbb{D}^n(c)=(\mathbb{B}^n(c),\hat{g}_c)$ the same ball equipped with the Riemannian metric $\hat{g}_c$ (i.e., the Poincar\'e ball model), where the curvature of the corresponding hyperbolic space is $-c,c>0$. Moreover, we denote $\bm{T}_{\bm{z}}\mathbb{D}^n(c)\subset \mathbb{R}^n$ as the tangent space at $\bm{z}\in \mathbb{D}^n(c)$. 

\subsection{Poincar\'e Ball Model}

Hyperbolic space admits multiple isometric models for representation, among which the most commonly used are the Poincar\'e ball model, the Poincar\'e half space model, the Klein model, the Lorentz (Hyperboloid) model, and the Hemisphere model \cite{beltrami1868teoria, cannon1997hyperbolic}. The Poincar\'e ball model is one of the most widely used models for representing hyperbolic geometry. The n-dimensional Poincar\'e ball describes hyperbolic space as a Riemannian manifold equipped with a Riemannian metric $\hat{g}_c$:

\begin{equation}
    \mathbb{D}^n(c) = \{\bm{z}\in\mathbb{C}^n \mid \Vert z\Vert<\tfrac{1}{\sqrt{c}},\quad c>0\}.
\end{equation}

The Riemannian metric is defined by $\hat{g}_c(\bm{z})=\lambda_c^2(\bm{z})\cdot g^E$, where $\lambda_c(\bm{z})=\frac{1}{1-c\Vert z\Vert^2}$ is the conformal factor and $g^E=\bm{I}_n$ is the Euclidean metric. Consequently, the Riemannian metric equips the tangent space with an inner product $\langle \bm{u},\bm{v} \rangle_{\bm{T}_{\bm{z}}\mathbb{D}^n(c)}=\bm{u}^\top \hat{g}_c(\bm{z}) \bm{v},\forall \bm{u},\bm{v}\in \bm{T}_{\bm{z}}\mathbb{D}^n(c)$  at any point $\bm{z}\in \mathbb{D}^n(c)$.

We detail several hyperbolic operations based on the M\"obius gyrovector spaces \cite{ungar1998pythagoras, ungar2008analytic} adopted in this paper as follows.

\begin{itemize}
    \item \textbf{Exponential Map at the Origin:} The exponential map takes a point $\bm{v}$ in the tangent space $\bm{T}_{\bm{z}}\mathbb{D}^n(c)$ as a velocity vector and maps it along the geodesic on $\mathbb{D}^n(c)$ to a corresponding point. We employ the exponential map to project Euclidean features into hyperbolic space, for which it suffices to consider $\bm{T}_{\bm{0}}\mathbb{D}^n(c)$:
    \begin{equation}
    \begin{aligned}
        \text{exp}_{\bm{0}}:& \bm{T}_{\bm{0}}\mathbb{D}^n(c) \to \mathbb{D}^n(c), \\&\bm{v} \mapsto \text{tanh}\left(\sqrt{c}\Vert \bm{v}\Vert\right)\frac{\bm{v}}{\sqrt{c}\Vert \bm{v}\Vert},
        \label{eq: exp}
    \end{aligned}
    \end{equation}
    \item \textbf{M\"obius Self-Mappings and Pseudo-Hyperbolic distance:}
    The Poincar\'e ball model possesses its own automorphisms (i.e., the M\"obius self-mappings) and pseudo-hyperbolic distance. In this work, we introduce a generalized formulation of M\"obius self-mappings defined on Poincar\'e balls with arbitrary curvature \cite{DBLP:conf/ijcai/YangFX23}. When $n=1$, the M\"obius self-mappings are equal to the M\"obius subtraction:
    \begin{equation}
        \varphi^c_{z_i}(z_j)=\frac{z_i-z_j}{1-cz_i^*z_j}, \quad z_i,z_j\in\mathbb{D}^1(c).
    \end{equation}
    While certain correction terms need to be introduced when $n>1$:
    \begin{equation}
       \bm{\varphi}^c_{\bm{z}_i}(\bm{z}_j)=\frac{\bm{z}_i-P^c_{\bm{z}_i}(\bm{z}_j)-s^c_{\bm{z}_i}Q^c_{\bm{z}_i}(\bm{z}_j)}{1-c\bm{z}_i^*\bm{z}_j},
    \end{equation}
    where $P^c_{\bm{z}_i}(\bm{z}_j)=0$ if $\bm{z}_i=\bm{0}$, otherwise $\frac{\bm{z}_i^*\bm{z}_j}{\Vert \bm{z}_i\Vert^2}\bm{z}_i$. Additionally,  $s^c_{\bm{z}_i}=\sqrt{1-c\Vert\bm{z}_i\Vert^2}$ and $Q^c_{\bm{z}_i}(\bm{z}_j)=\bm{z}_j-P^c_{\bm{z}_i}(\bm{z}_j)$.
        
\end{itemize}

\section{Adaptive Hyperbolic Kernels}
In this section, we introduce our proposed adaptive hyperbolic kernel framework. We begin by formulating a curvature-aware generalization of the de Branges–Rovnyak space, which serves as the theoretical foundation for our kernel design. Based on this construction, we then present a class of hyperbolic kernels that support curvature flexibility and task-adaptive modulation.

\subsection{Curvature-aware De Branges-Rovnyak Space}

We first introduce the de Branges-Rovnyak space, an RKHS defined via a subtractive kernel with a multiplier function~\cite{sautel2022some}, and then generalize it to a curvature-aware formulation on Poincar\'e balls of arbitrary curvature, enabling adaptive geometric representation.
\begin{definition}
    A Hilbert space $\mathcal{H}_n^{\bm{b}}$ on $\mathbb{B}^n$ is a de Branges-Rovnyak space if and only if it is the RKHS associated with the following kernel function:
    \begin{equation}
        k^{\bm{b}}(\bm{z}_i,\bm{z}_j)=\frac{1-\bm{b}(\bm{z}_i)^*\bm{b}(\bm{z}_j)}{1-\bm{z}_i^*\bm{z}_j},
    \end{equation}
    for some {{$\bm{b}:\mathbb{B}^n\to \mathcal{D},\bm{z}\mapsto(b_1(\bm{z}),b_2(\bm{z}),\ldots,b_r(\bm{z}))^\top$ and $r\in\mathbb{N}\cup\{\infty\}$}}. $\mathcal{D}$ denotes an $r$-dimensional Hilbert space.
\end{definition}
To ensure the de Branges-Rovnyak space is well-defined,  the kernel $k^{\bm{b}}$ should be pd ~\cite{aronszajn1950theory}.
This condition depends on the choice of the multiplier function $\bm{b}$, which must belong to the following multiplier space: 
\begin{equation}
\begin{aligned}
\mathcal{M}(\mathcal{H}^2_n \otimes \mathcal{D}, \mathcal{H}^2_n)
= \Big\{ &\bm{f} : \mathbb{B}^n \to \mathcal{D} \;\Big|\; \bm{f} \bm{h} \in \mathcal{H}^2_n \\
&\; \forall \bm{h} \in \mathcal{H}^2_n \otimes \mathcal{D} \Big\},
\end{aligned}
\end{equation}
where $\otimes$ denotes the tensor product over the complex field and $\mathcal{H}^2_n$ is the Drury Arveson Hardy space \cite{drury1978generalization,Arveson1998}. Here $\mathcal{H}^2_n \otimes \mathcal{D}$ is canonically identified with the vector-valued function space of dimension $r$ with components residing in the Drury Arveson Hardy space. Under this framework, the bounded multiplier $\bm{f}$ acts on a function $\bm{h}$ via pointwise multiplication $(\bm{f} \bm{h})(\bm{z})=\sum_{i=1}^rf_i(\bm{z})h_i(\bm{z}),\quad \bm{z}\in\mathbb{B}^n,h_i\in\mathcal{H}^2_n$.

The following proposition characterizes the condition under which the de Branges-Rovnyak kernel $k^{\bm{b}}$ remains pd:
\begin{proposition} \cite{sautel2022some}
A function $\bm{b}$ belongs to a multiplier in $\mathcal{M}(\mathcal{H}^2_n \otimes \mathcal{D}, \mathcal{H}^2_n)$ if and only if:
    \begin{equation}
        \frac{\varepsilon^2\bm{I}-\bm{b}(\bm{z}_i)^*\bm{b}(\bm{z}_j)}{1-\bm{z}_i^*\bm{z}_j}\succeq 0,
    \end{equation}
    for some $\varepsilon>0$. The infimum of such constants $\varepsilon$ is known as the multiplier norm of $\bm{b}$, denoted by $\Vert\bm{b}\Vert_{\mathcal{M}(\mathcal{H}^2_n \otimes \mathcal{D}, \mathcal{H}^2_n)}$.
    \label{prop: mul}
\end{proposition}
This proposition implies that the kernel $k^{\bm{b}}$ is pd if and only if the corresponding multiplier function $\bm{b}\in\mathcal{M}(\mathcal{H}^2_n \otimes \mathcal{D}, \mathcal{H}^2_n)$ with $\Vert\bm{b}\Vert_{\mathcal{M}(\mathcal{H}^2_n \otimes \mathcal{D}, \mathcal{H}^2_n)}\leq 1$. 
\subsubsection{Contractive Space Structure}

The de Branges-Rovnyak space, using the Drury–Arveson Hardy space   $\mathcal{H}_n^2$ as a bridge, naturally allows an isometric embedding into hyperbolic geometry. Here, the Drury–Arveson Hardy space is the canonical RKHS on the unit ball $\mathbb{B}^n$ with reproducing kernel $k^{da}(\bm{z}_i,\bm{z}_j)=\frac{1}{1-\bm{z}_i^*\bm{z}_j}$, and isometric to the Poincar\'e ball~\cite{arcozzi2007diameter,DBLP:conf/ijcai/YangFX23}.

Since $k^{\bm{b}}$ is pd, it follows that the corresponding RKHS $\mathcal{H}_n^{\bm{b}}$ is a subspace of the Drury Arveson Hardy space by the following proposition:
\begin{proposition} \cite{sautel2022some}
   Given two kernel functions $k_1$ and $k_2$, and their corresponding RKHS $\mathcal{H}_1$ and $\mathcal{H}_2$, $\mathcal{H}_1\subseteq\mathcal{H}_2$ if and only if $\varepsilon^2k_2-k_1\succeq 0$ for some $\varepsilon>0$.
\end{proposition}

\noindent Setting $\varepsilon=1$, $k_1=k^{\bm{b}}$, and $k_2=k^{da}$, we obtain $k^{da}(\bm{z}_i,\bm{z}_j)-k^{\bm{b}}(\bm{z}_i,\bm{z}_j)=\frac{\bm{b}(\bm{z}_i)^*\bm{b}(\bm{z}_j)}{1-\bm{z}_j^*\bm{z}_i}=(\bm{b}(\bm{z}_i)^*\bm{b}(\bm{z}_j))k^{da}(\bm{z}_i,\bm{z}_j):=k'(\bm{z}_i,\bm{z}_j)\cdot k^{da}(\bm{z}_i,\bm{z}_j)$. 
Since both \(k'\) and \(k^{da}\) are pd, \(k^{da}\) is also pd by the Schur Product Theorem~\cite{Schur1911},  which implies
\(\mathcal{H}_n^{\bm b}\subseteq \mathcal{H}_n^2\).

As a contractive subspace of the Drury Arveson Hardy space, $\mathcal{H}_n^{\bm{b}}$ isometrically embeds into a subspace of the Poincar\'e ball, allowing hyperbolic kernels defined on it to map hierarchical data with minimal distortion. Besides, this inherent contraction of the RKHS reduces the scale of hypothesis space and serves as an implicit regularizer, accelerating convergence and improving generalization.

\subsubsection{Generalization to Arbitrary Curvature}
To extend this isometry to Poincar\'e balls of any curvature, we introduce the curvature-aware de Branges–Rovnyak kernel:
\begin{equation}
 k^{\bm{b}}_c(\bm{z}_i,\bm{z}_j)=\frac{1-c\bm{b}(\bm{z}_i)^*\bm{b}(\bm{z}_j)}{1-c\bm{z}_i^*\bm{z}_j}  \quad \bm{z}_i, \bm{z}_j\in\mathbb{D}^n(c),
\end{equation}
where $\bm{b}:\mathbb{D}^n(c)\to \mathcal{D}$. This kernel induces a RKHS, namely the curvature-aware de Branges-Rovnyak space denoted by $\mathcal{H}_n^{\bm{b}}(c)$. Under the coordinate rescaling $\bm{z}\to\sqrt{c}\bm{z}$, this kernel and its RKHS coincide exactly with the standard Poincar\'e ball model at curvature $-1$ (i.e.\ c=1), so all positive definiteness and contractive-embedding properties carry over without modification. 
In this setting, the multiplier norm condition positive definiteness constraint becomes
$\bm{b}\in \mathcal{M}(\mathcal{H}^2_n(c)\otimes \mathbb{D}^n(c),\mathcal{H}^2_n(c))$ with $\sqrt{c}\Vert\bm{b}\Vert_{\mathcal{M}(\mathcal{H}^2_n(c)\otimes \mathbb{D}^n(c),\mathcal{H}^2_n(c))}\leq 1$, where $\mathcal{H}^2_n(c)$ is the RKHS induced by the kernel $k_c^{da}(\bm{z}_i,\bm{z}_j)=\frac{1}{1-c\bm{z}_i^*\bm{z}_j}$, ensuring $k_c^{\bm{b}}$  remains a valid kernel.

This family of curvature-aware kernels thus provides a seamless way to extend the de Branges–Rovnyak construction to Poincar\'e balls of arbitrary curvature, while preserving low-distortion embedding and implicit regularization.

\subsection{Proposed Hyperbolic Kernels}
\subsubsection{Design of the Multiplier $\bm{b}$}
We modulate the original hyperbolic features via M\"obius self-mappings, enabling the kernel to adaptively enhance or suppress pairwise similarities. 
Concretely, we define: 
\begin{equation}
\begin{aligned}
    \bm{b}(\bm{z})&=\frac{1}{2}\sum_{i=1}^{m}w_i\left(\bm{\varphi}^c_{\bm{a}_i}(\bm{z})+\bm{\varphi}^c_{-\bm{a}_i}(\bm{z})\right)\\
&=\sum_{i=1}^{m}w_i\frac{(c\bm{a}_i^*\bm{z})\bm{a}_i-P_{\bm{a}_i}^c(\bm{z})-s^c_{\bm{a}_i}Q^c_{\bm{a}_i}(\bm{z})}{1-(c\bm{a}_i^*\bm{z})^2},
\end{aligned}
\label{eq: multiplier_def1}
\end{equation}
where each weight $w_i>0$ satisfies $\sum_iw_i=1$ and $\bm{a}_i\in\mathbb{D}^n(c)$ are learnable hyperbolic poles. When \( n = 1 \), this reduces to:
\begin{equation}
    b(z)=\frac{1}{m}\sum_{i=1}^{m}w_i\frac{(ca_i^*z)a_i}{1-(ca_i^*z)^2}.
    \label{eq: multiplier_def2}
\end{equation} 
By the convex combination of M\"obius mappings, the resulting de Branges–Rovnyak kernel: 
\begin{equation}
k^{\bm{b}}_c(\bm{z}_i,\bm{z}_j)=\frac{1-c\bm{b}(\bm{z}_i)^*\bm{b}(\bm{z}_j)}{1-\bm{z}_i^*\bm{z}_j}
\end{equation}
is pd. We clarify its positive definiteness by the following lemma and theorem:
\begin{lemma}   Given $\tilde{\bm{b}}(\bm{z})=\bm{\varphi}_{\bm{a}}^c(\bm{z})$, the kernel $k^{\tilde{\bm{b}}}_c(\bm{z}_i,\bm{z}_j)=\frac{1-c\tilde{\bm{b}}(\bm{z}_i)^*\tilde{\bm{b}}(\bm{z}_j)}{1-c\bm{z}_i^*\bm{z}_j}$ on $\mathbb{D}^n(c)$ is pd for any $\bm{a}\in\mathbb{D}^n(c)$. \label{lem: pos}
\end{lemma}
\begin{theorem}    Given $\hat{\bm{b}}(\bm{z})=\sum_{i=1}^{m}w_i\bm{\varphi}^c_{\bm{a}_i}(\bm{z})$, $\forall \bm{a}_1,\ldots \bm{a}_m\in\mathbb{D}^n(c)$ and $w_i>0,\sum_{i=1}^mw_i=1$ , the corresponding kernel $k^{\hat{\bm{b}}}_c(\bm{z}_i,\bm{z}_j)=\frac{1-c\hat{\bm{b}}(\bm{z}_i)^*\hat{\bm{b}}(\bm{z}_j)}{1-c\bm{z}_i^*\bm{z}_j}$ is always pd for any $m\geq 1$.\label{thm: pos}
\end{theorem}
Moreover, $\bm{b}$ satisfies {{$\bm{b}(\bm{0})=\bm{0}$, $\bm{b}(-\bm{z})=-\bm{b}(\bm{z})$}}, and  {{$k_c^{\bm{b}}(-\bm{z}_i,-\bm{z}_j)=k_c^{\bm{b}}(\bm{z}_i,\bm{z}_j)$}}. These symmetry properties preserve the structure of hyperbolic geometry, ensure geometric consistency, and yield stable, interpretable similarity measures. They also constrain the hypothesis space, acting as an implicit regularizer.

\subsubsection{Design of Hyperbolic Kernels}
Following the work of \cite{DBLP:conf/ijcai/YangFX23}, we interpret hyperbolic feature similarity as inner products in the de Branges–Rovnyak RKHS. For each data point $z$, let $\hat{k}^{\bm{b}}_{\bm{z}_i}\in \mathcal{H}_n^{\bm{b}}(c)$ denote the functional in de Branges–Rovnyak space. Then the basic de Branges–Rovnyak kernel is:
\begin{equation}
k^{\bm{b}}_c(\bm{z}_i,\bm{z}_j)=\langle \hat{k}^{\bm{b}}_{\bm{z}_i}, \hat{k}^{\bm{b}}_{\bm{z}_j}\rangle=\frac{1-c\bm{b}(\bm{z}_i)^*\bm{b}(\bm{z}_j)}{1-\bm{z}_i^*\bm{z}_j}.
\end{equation}
As in Euclidean kernel design, one may simply replace $\bm{z}$ by $\hat{k}^{\bm{b}}_{\bm{z}_i}$ in standard kernels to obtain four adaptive hyperbolic variants:
 \\adaptive hyperbolic linear kernel \textbf{(AHL)}: {{$k^{\text{AHL}}(\bm{z}_i,\bm{z}_j)= k^{\bm{b}}_c(\bm{z}_i,\bm{z}_j)$}},\\
adaptive hyperbolic polynomial Kernel \textbf{(AHPoly)}: {{$k^{\text{AHPoly}}(\bm{z}_i,\bm{z}_j)= \left(\langle\hat{k}^{\bm{b}}_{\bm{z}_i},\hat{k}^{\bm{b}}_{\bm{z}_j}\rangle+b\right)^d,b,d>0$}},\\
  adaptive hyperbolic RBF kernel \textbf{(AHRBF)}:
        {{$k^{\text{AHRBF}}(\bm{z}_i,\bm{z}_j)= \text{exp}\left(-\frac{\Vert \hat{k}^{\bm{b}}_{\bm{z}_i}- \hat{k}^{\bm{b}}_{\bm{z}_j}\Vert^2}{2\tau^2}\right), \tau>0$}},
        \\ adaptive hyperbolic Laplacian kernel \textbf{(AHLap)}:
        {{$k^{\text{AHLap}}(\bm{z}_i,\bm{z}_j)= \text{exp}\left(-\frac{\Vert \hat{k}^{\bm{b}}_{\bm{z}_i}- \hat{k}^{\bm{b}}_{\bm{z}_j}\Vert}{\tau}\right), \tau>0$}}.\\
These kernels are straightforward generalizations, while our main contribution lies in the following 
adaptive hyperbolic radial kernel (\textbf{AHRad}).

To construct the AHRad, we first define a base kernel as the squared cosine similarity of the normalized  representers in the de Branges-Rovnyak space as: 
\begin{equation}
k^{\text{base}}(\bm{z}_i,\bm{z}_j)=\left\Vert\left\langle \frac{\hat{k}^{\bm{b}}_{\bm{z}_i}}{\Vert\hat{k}^{\bm{b}}_{\bm{z}_i}\Vert}, \frac{\hat{k}^{\bm{b}}_{\bm{z}_j}}{\Vert\hat{k}^{\bm{b}}_{\bm{z}_j}\Vert}\right\rangle\right\Vert^2
\end{equation}
which satisfies $0\leq k^{\text{base}}(\bm{z}_i,\bm{z}_j)< 1$. Subsequently, we can express the  AHRad as a nonnegative power series~\cite{a54edda841764a338ae25d853527949e,DBLP:journals/pami/JayasumanaHSLH15}:
\begin{equation}
\begin{aligned}
k^{\text{AHRad}}(\bm{z}_i,\bm{z}_j)&=\sum_{l=0}^{\infty}\alpha_l\left(k^{\text{base}}(\bm{z}_i,\bm{z}_j)\right)^{l}\\&+\sum_{l=-1,-2}\alpha_{l}k_{l}(\bm{z}_i,\bm{z}_j),
\end{aligned}
\end{equation}
where $\alpha_l\geq0$ and $\sum_{l=-2}^{\infty}\alpha_i<\infty$. The terms involving 
$k_{-1}$ and $k_{-2}$ are only needed to force exact self‐similarity in the infinite expansion and can be dropped in practice since $k^{\text{AHRad}}(\bm{z}_i,\bm{z}_i)>k^{\text{AHRad}}(\bm{z}_i,\bm{z}_j)$ holds strictly when $\bm{z}_i\neq\bm{z}_j$ even without these two terms. Moreover, since $\Vert k^{\text{base}}(\bm{z}_i,\bm{z}_j)\Vert< 1$, the remainder beyond $l=K$ decays geometrically, making this finite expansion both computationally efficient and numerically stable. Therefore, the infinite series can be implemented as: 
\begin{equation}
k^{\text{AHRad}}(\bm{z}_i,\bm{z}_j)=\sum_{l=0}^{K}\alpha_l\left(k^{\text{base}}(\bm{z}_i,\bm{z}_j)\right)^{l}.
\label{eq:AHRad}
\end{equation}
 
By construction, $k^{\text{AHRad}}$ remains pd. Besides, it aligns with the multi-kernel learning strategy, thus it can accommodate higher-order feature interactions.

\section{Experiments}

We conduct three groups of experiments, including few-shot learning and zero-shot learning on image datasets, and the semantic textual similarity evaluation (STS) task on a text dataset, to demonstrate the superiority of our proposed hyperbolic kernels. The few-shot and zero-shot tasks were run on an NVIDIA RTX 3090 Ti, while the STS task was run on an NVIDIA RTX 4090.

For each task, Euclidean features are first projected onto the Poincar\'e ball using the exponential map in Eq. \eqref{eq: exp} or an alternative mapping \cite{guo2022clipped} defined by {$\text{CLIP}_{\beta, \varepsilon}(\bm{x})=\beta \text{min}\left\{1,\frac{1-\varepsilon}{\sqrt{c}\Vert\bm{x}\Vert}\right\}\bm{x}$},
where $\varepsilon\!\in\!(0,1)$ controls the clipping radius and $\beta$ is an additional scaling factor.

\subsection{Few-Shot Learning}

\subsubsection{Experimental Framework}
In this section, we consider the task of few-shot image classification, to learn a model capable of rapidly generalizing to novel categories using only a few labeled samples. We adopt Prototypical networks \cite{snell2017prototypical}, a classic metric learning-based approach as our backbone, and follow the kernel learning paradigm established by Fang et al.\ \cite{DBLP:journals/ijcv/FangHLP23}. Within this framework, we embed our kernel function into the loss function, replacing the original metric.

We compare our proposed kernels against nine hyperbolic kernels, including the Poincar\'e Linear kernel (\textbf{PTang}), Poincar\'e RBF kernel (\textbf{PRBF}), Poincar\'e Laplace kernel (\textbf{PLap}), Poincar\'e Binomial kernel (\textbf{PBin}), and Poincar\'e Radial kernel (\textbf{PRad})~\cite{DBLP:journals/ijcv/FangHLP23}, as well as the Curvature-aware Hyperbolic Linear kernel (\textbf{CHL}), Curvature-aware Hyperbolic Polynomial kernel (\textbf{CHPoly}), Curvature-aware Hyperbolic RBF kernel (\textbf{CHRBF}) and Curvature-aware Hyperbolic Laplacian kernel (\textbf{CHLap})~\cite{DBLP:conf/ijcai/YangFX23}. We adopt as our baseline a method based on hyperbolic embeddings, where the hyperbolic geodesic distance is used as the similarity metric \cite{DBLP:conf/cvpr/KhrulkovMUOL20}.

\subsubsection{Datasets and Evaluation}

We evaluate our kernel functions on two image datasets, \textbf{CUB}~\cite{CUB} and \textbf{mini-ImageNet}~\cite{deng2009imagenet}, with detailed dataset

\begin{table}[t]
\centering
\setlength{\tabcolsep}{1.4mm}
\begin{tabular}{l|c|cc|cc}
\toprule
\multirow{2}{*}{Methods} & \multirow{2}{*}{Backbone} & \multicolumn{2}{c|}{CUB} & \multicolumn{2}{c}{\emph{mini}ImageNet} \\
\cmidrule{3-6}
& & 5w1s & 5w5s & 5w1s & 5w5s \\
\midrule
PTang & Conv-4 & $60.1_{0.3}$ & $82.0_{0.2}$ & $54.0_{0.2}$ & $73.1_{0.2}$ \\
PRBF & Conv-4 & $61.4_{0.2}$ & $82.7_{0.2}$ & $54.4_{0.2}$ & $73.1_{0.2}$ \\
PLap & Conv-4 & $62.9_{0.2}$ & $81.7_{0.2}$ & $53.1_{0.2}$ & $71.3_{0.2}$ \\
PBin & Conv-4 & $62.6_{0.2}$ & $83.0_{0.1}$ & $53.4_{0.2}$ & $72.6_{0.2}$ \\
PRad & Conv-4 & $65.6_{0.2}$ & $82.4_{0.2}$ & $53.6_{0.2}$ & $72.9_{0.2}$ \\
\midrule
CHL & Conv-4 & \textbf{65.8}$_{0.2}$ & $82.9_{0.2}$ & $53.6_{0.2}$ & $72.6_{0.2}$ \\
CHPoly & Conv-4 & $62.3_{0.2}$ & $80.3_{0.2}$ & $52.6_{0.2}$ & $71.3_{0.2}$ \\
CHRBF & Conv-4 & $56.7_{0.2}$ & $79.9_{0.2}$ & $53.5_{0.2}$ & $71.9_{0.2}$ \\
CHLap & Conv-4 & $56.9_{0.2}$ & $79.1_{0.2}$ & $52.7_{0.2}$ & $71.5_{0.2}$ \\
\midrule
AHL & Conv-4 & $61.2_{0.2}$ & $82.7_{0.1}$ & $52.5_{0.2}$ & $71.7_{0.2}$ \\
AHPoly & Conv-4 & $62.0_{0.2}$ & $82.2_{0.2}$ & $53.8_{0.2}$ & $71.7_{0.2}$ \\
AHRBF & Conv-4 & $60.6_{0.2}$ & $82.2_{0.2}$ & $52.4_{0.2}$ & $72.0_{0.2}$ \\
AHLap & Conv-4 & $60.9_{0.2}$ & $81.6_{0.2}$ & $53.0_{0.2}$ &  $71.5_{0.2}$\\
AHRad & Conv-4 & $63.9_{0.2}$ & \textbf{83.3}$_{0.2}$ & \textbf{54.6}$_{0.2}$ & \textbf{73.2}$_{0.2}$ \\
\midrule
Baseline & Conv-4 & $59.6_{0.2}$ & $78.3_{0.2}$ & $52.7_{0.2}$ & $71.7_{0.2}$ \\
\bottomrule
\end{tabular}
\caption{Few-shot learning experimental results. Backbone denotes the backbone feature extraction model adopted. "CwMs" represents C-way M-shot. The metric is mean classification accuracy (ACC\%, $\uparrow$). The best result for each dataset and experimental setting is highlighted in \textbf{bold}. Subscripts indicate the 95\% confidence interval.}
\label{tab:fsl}
\end{table}
\noindent configurations provided in the supplementary material. We adopt the mean classification accuracy (ACC) as our evaluation metric, and conduct few-shot learning experiments under the settings of $5$-way, $\{1, 5\}$-shot on both datasets. 

\subsubsection{Experimental Results}
According to Table \ref{tab:fsl}, our proposed AHRad achieves the best performance on the mini-ImageNet dataset as well as the 5-way 5-shot task on the CUB dataset. For the 5-way 1-shot task, it ranks just behind the PRad and CHL. These results demonstrate that our kernel exhibits strong competitiveness compared to other kernel functions in few-shot learning. Notably, other adaptive hyperbolic kernels also exhibit comparable or superior performance compared to their corresponding curvature-aware hyperbolic (\textbf{CH}) counterparts, as a reference.

\subsection{Zero-Shot Learning}

\subsubsection{Experimental Framework}
We focus on cross-modal zero-shot learning aiming to align semantic and visual modalities on seen classes, and to recognize images from unseen classes during inference \cite{akata2015label,xian2016latent}. Following the kernel learning paradigm in the work of Fang et al.\ \cite{DBLP:journals/ijcv/FangHLP23}, the original Euclidean metric is replaced by our proposed kernels.

We adopt the same comparison methods as those used in the few-shot learning section. Besides, we also adopt additional typical zero-shot learning methods for comparison, including \textbf{LATEM}~\cite{xian2016latent}, \textbf{DEVISE}~\cite{frome2013devise}, \textbf{DEM}~\cite{zhang2017learning},  \textbf{ALE}~\cite{akata2015label}, \textbf{SP-AEN}~\cite{DBLP:conf/cvpr/ChenZ00C18}, \textbf{CRnet}~\cite{zhang2019co} and \textbf{Kai et al.}~\cite{li2019rethinking}. While the implementation of the baseline follows that of Fang et al.\ \cite{DBLP:journals/ijcv/FangHLP23}.

\begin{table}[t]
\centering
\setlength{\tabcolsep}{0.8mm}
\begin{tabular}{l|ccc|ccc|ccc}
\toprule
\raisebox{-1ex}{Methods} & \multicolumn{3}{c|}{CUB} & \multicolumn{3}{c|}{AWA1} & \multicolumn{3}{c}{AWA2} \\
\cline{2-10}
& \raisebox{-0.5ex}{U} & \raisebox{-0.5ex}{S} & \raisebox{-0.5ex}{HM} & \raisebox{-0.5ex}{U} & \raisebox{-0.5ex}{S} & \raisebox{-0.5ex}{HM} & \raisebox{-0.5ex}{U} & \raisebox{-0.5ex}{S} & \raisebox{-0.5ex}{HM} \\
\midrule
LATEM & 15.2 &57.3&24.0&7.3& $71.7$&13.3&11.5&77.3&20.0\\
 DEVISE  &  23.8&53.0&32.8&13.4&68.7&22.4 &17.1 &74.7&27.8\\
DEM &19.6 &57.9&29.2& 32.8&84.7&47.3 &30.5&86.4&45.1 \\
 ALE &23.7 &62.8 &34.4&16.8&76.1&27.5 &14.0&81.8&23.9\\
SP-AEN  &34.7 &70.6 &46.6 &-&-&-&23.3& 90.9 &37.1\\
CRnet &45.5 &56.8 &50.5 &58.1 &74.7&65.4&52.6&78.8&63.1\\
Kai et al.  &47.4&47.6  & 47.5 &62.7&77.0&69.1&56.4&81.4&66.7 \\
\midrule
PTang & 40.8 & 58.1 & 47.9 & 52.3 & 85.2 & 64.8 & 37.1 & 88.5 & 52.3 \\
PRBF & 44.6 & 57.8 & 50.3 & 59.0 & 84.6 & 69.5 & 42.9 & 89.5 & 57.9 \\
PLap & 46.2 & 56.1 & 50.7 & 60.7 & 83.5 & 70.3 & 54.1 & 87.1 & 66.7 \\
PBin & 39.8 & 56.9 & 46.8 & 43.7 & 88.9 & 58.6 & 39.8 & 90.5 & 55.4 \\
PRad & 45.8 & 57.6 & 51.0 & 60.2 & 86.7 & 71.1 & 48.2 & 90.3 & 62.8 \\
\midrule
CHL & 43.3 & 58.3 & 49.7 & 51.2 & 84.7 & 63.8 & 44.5 & 90.8 & 59.8 \\
CHPoly & 41.7 & 58.9 & 48.8 & 51.3 & 85.4 & 64.1 & 42.2 & 90.9 & 57.6 \\
CHRBF & 45.0 & 56.7 & 50.1 & 56.3 & 82.7 & 67.0 & 45.1 & 90.1 & 60.1 \\
CHLap & 45.2 & 56.1 & 50.1 & 53.4 & 88.9 & 66.7 & 44.9 & 90.9 & 60.1 \\
\midrule
AHL & 46.2 & 56.1 & 50.7 & 55.4 & 85.3 &67.2  & 49.4 & 89.4& 63.6 \\
AHPoly & 49.0 & 52.8 & 50.8 & 58.8 & 85.3 & 69.6 & 46.2 &87.6 &60.5  \\
AHRBF & 44.7 & 58.3 & 50.6 & 59.3 & 85.2 & 69.9 & 48.3 & 88.6 & 62.5 \\
AHLap & 47.4 & 55.3 & 51.0 & 56.8 & 82.5 & 67.3 & 47.1 & 86.5 & 61.0 \\
AHRad & 49.0 & 54.0 & \textbf{51.4} & 64.9 & 83.7 & \textbf{73.1} & 69.3 & 84.0 & \textbf{75.9} \\
\midrule
Baseline &18.6&44.6&26.3&29.8&76.4&42.9 &25.5&76.4&38.2\\
\bottomrule
\end{tabular}
\caption{Zero-shot learning experimental results. Backbone denotes the backbone feature extraction model adopted. "U" and "S" denote the mean classification accuracy (ACC\%, $\uparrow$) on the seen and unseen datasets, respectively. "HM" is the harmonic mean of "U" and "S". The best result for each dataset and experimental setting is highlighted in \textbf{bold}.}
\label{tab:zeroshot}
\end{table}

\subsubsection{Datasets and Evaluation}
We validated the zero-shot capability of our kernel function on three image datasets: \textbf{CUB} \cite{CUB}, \textbf{AWA1} \cite{lampert2013attribute}, and \textbf{AWA2} \cite{ALE_Akata_PAMI}. Details of the datasets are also provided in the supplementary material (please refer to the link of extended version). We evaluated the model's average classification accuracy on both the seen and unseen data, denoted as $\mathrm{S}$ and $\mathrm{U}$ respectively, to reflect the model's learning and generalization ability. Additionally, we computed the harmonic mean $\mathrm{HM}$ of $\mathrm{S}$ and $\mathrm{U}$ as a comprehensive metric for the model's overall capability.

\subsubsection{Experimental Results}

Our comparison includes both kernel-based methods—using different kernels within a unified framework—and other few-shot learning approaches that adopt distinct paradigms. As shown in Table \ref{tab:zeroshot}, AHRad attains the best performance on CUB, AWA1, and AWA2, surpassing the second-best method by 0.4\%, 2.0\%, and 9.2\%, respectively. These results show that our method is at least competitive on CUB and substantially outperforms the second-best approaches on AWA1 and AWA2. Moreover, it consistently yields the highest accuracy on unseen classes, with notable gains over the second-best method, indicating markedly improved generalization. In addition, the other adaptive hyperbolic kernels also achieve comparable or superior performance relative to their curvature-aware counterparts.

\subsection{Semantic Textual Similarity Evaluation}

\begin{table}[t]
\centering
\setlength{\tabcolsep}{4mm}
\begin{tabular}{l|c|c}
\toprule
Methods & Backbone & sup  \\
\midrule
PTang   & Bert-base-uncased & 84.36  \\
PRBF    & Bert-base-uncased & 84.84  \\
PLap    & Bert-base-uncased & 84.63  \\
PBin    & Bert-base-uncased & 84.06  \\
PRad    & Bert-base-uncased & 84.53  \\
\midrule
CHL     & Bert-base-uncased & 84.70  \\
CHPoly  & Bert-base-uncased & 84.63 \\
CHRBF   & Bert-base-uncased & 83.98 \\
CHLap   & Bert-base-uncased & 84.74  \\
\midrule
AHL     & Bert-base-uncased & 84.27  \\
AHPoly  & Bert-base-uncased & 84.68  \\
AHRBF   & Bert-base-uncased & 84.33 \\
AHLap   & Bert-base-uncased & 84.48  \\
\midrule
AHRad     & Bert-base-uncased & \textbf{85.16} \\
\midrule
Baseline & Bert-base-uncased & 84.24  \\
\bottomrule
\end{tabular}
\caption{Semantic textual similarity evaluation experimental results. Backbone denotes the backbone feature extraction model. "Sup" indicates experiments conducted under the supervised contrastive learning paradigm. The metric is Spearman's correlation coefficient (\%, $\uparrow$). The best result for each dataset and experimental setting is highlighted in \textbf{bold}.}
\label{tab:sts}
\end{table}

\subsubsection{Experimental Framework}

In this section, we design our experiments on contrastive learning-based semantic textual similarity evaluation. This task requires the model to recognize entailment or contradiction relationships between a premise sentence and a hypothesis sentence, assigning corresponding similarity scores. Specifically, we employ pretrained $\textbf{BERT}_{\text{base}}$ \cite{devlin2019bert} (uncased) as our base model. Following the \textbf{SimCSE} learning framework \cite{DBLP:conf/emnlp/GaoYC21}, we perform supervised pre-training and subsequently evaluate it on the semantic textual similarity benchmark (STS-B) task. We also adopt the same comparison methods as those used in the few-shot learning section. Our baseline is the original version of \textbf{SimCSE}, which employs cosine similarity as the metric.

\begin{figure}[t]
\centering
\includegraphics[width=0.9\columnwidth]{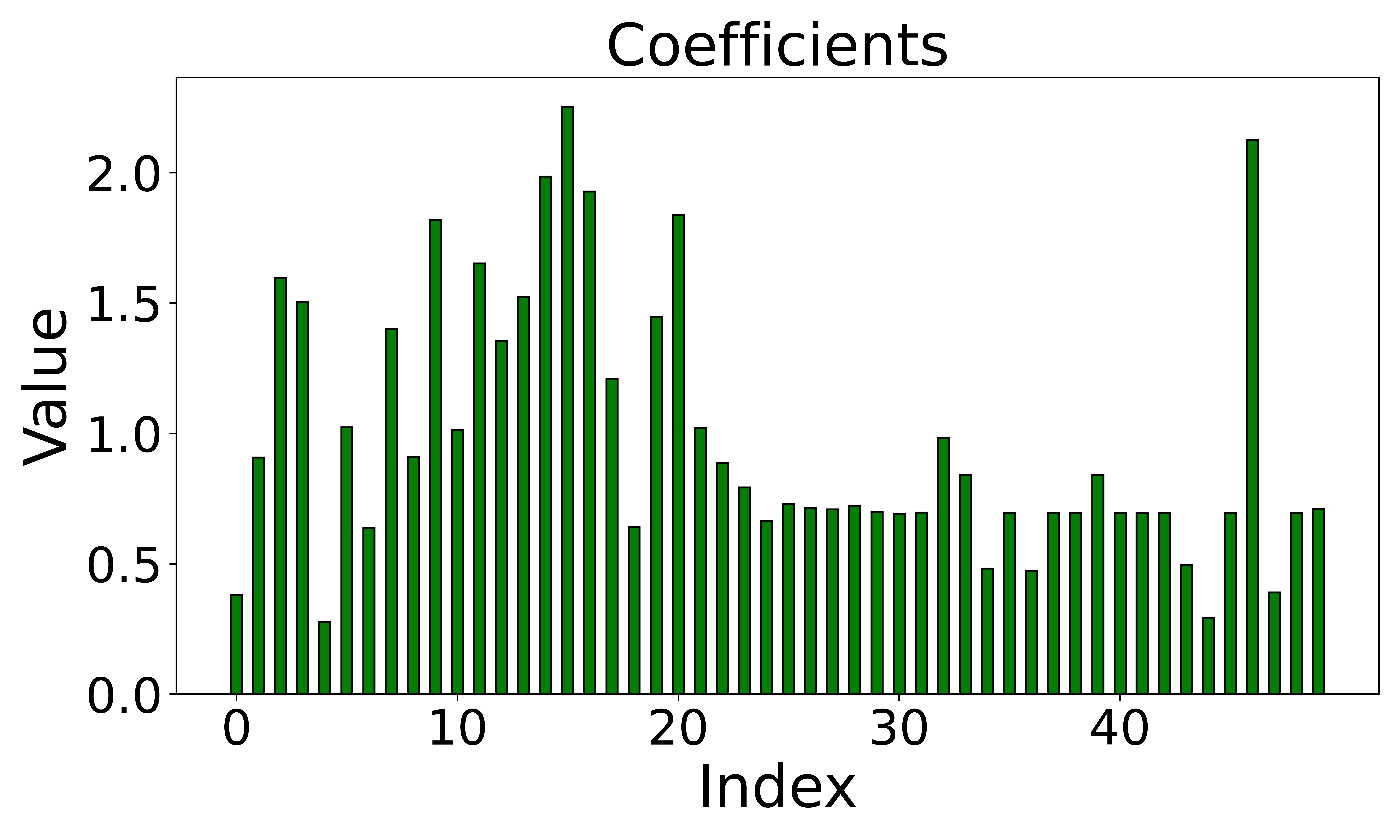}
\caption{Bar chart visualization of the parameters in Eq. \ref{eq:AHRad}, where the x-axis represents the index $l$ of the coefficients, and the y-axis denotes the magnitude of $a_l$}
\label{fig2}
\end{figure}

\subsubsection{Datasets and Evaluation}
For the pre-training stage, we adopt the hybrid dataset constructed from \textbf{MNLI} \cite{DBLP:conf/naacl/WilliamsNB18} and \textbf{SNLI} \cite{DBLP:conf/emnlp/BowmanAPM15}; for the evaluation stage, we adopt the \textbf{STS-B} dataset \cite{DBLP:conf/semeval/CerDALS17}. Our evaluation metric is the Spearman correlation between the similarity scores computed by the model and the human-annotated ground-truth scores.

\subsubsection{Experimental Results}
The baseline employs cosine similarity over Euclidean features, which is essentially a Euclidean kernel. In comparison, all hyperbolic kernels—except for \textbf{PBin} and \textbf{CHRBF}—consistently outperform the baseline, indicating that hyperbolic space can embed textual data with low distortion. According to table \ref{tab:sts}, our proposed \textbf{AHRad} achieves the best performance in this experiment, improving the correlation coefficient by 0.92 over the baseline, by 0.32 over the best-performing Poincar\'e kernel, and by 0.42 over the best-performing curvature-aware hyperbolic kernel. Notably, according to the original SimCSE results \cite{DBLP:conf/emnlp/GaoYC21}, upgrading the backbone network to the \textbf{RoBERTa}\textsubscript{base} model pretrained with a larger-scale corpus leads to an improvement of 1.58 over the baseline, suggesting that our kernel function can, to some extent, enhance the feature representation and serve as a lightweight alternative to larger models. Besides, other adaptive hyperbolic kernels also exhibit comparable or superior performance compared to their corresponding curvature-aware hyperbolic counterparts.

\section{Further Studies}

\begin{figure}[t]
\centering
\includegraphics[width=0.9\columnwidth]{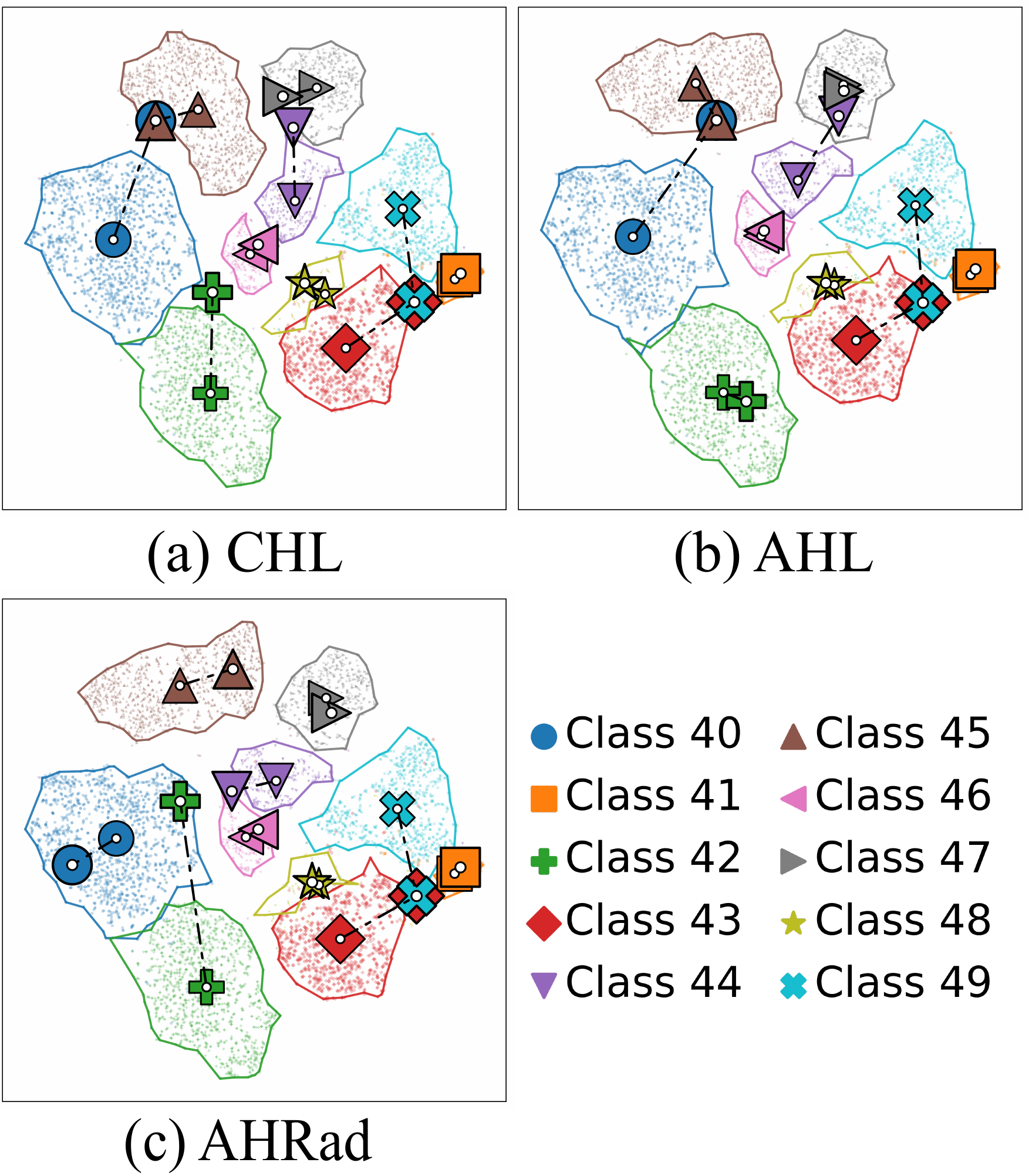}
\caption{Visualization of extracted features (visual and semantic) in the zero-shot learning setting on the AWA2 unseen dataset.}
\label{fig3}
\end{figure}

\subsection{Coefficients Distribution of AHRad}
In this section, we conduct zero-shot learning experiments on the CUB dataset and provide a visualization of the linear combination coefficients $a_l$ of the trained AHRad kernel. The results are presented in Figure \ref{fig2}.

This aims to offer an intuitive understanding of the multi-kernel learning framework of AHRad. Specifically, we visualize the coefficients $a_l$ corresponding to the first 50 terms in Equation \eqref{eq:AHRad} in our implementation. As shown, the low-order terms with small indices exhibit greater variation in their coefficients, while the high-order terms tend to have more stable coefficients. This indicates that lower-order components play a more significant role in shaping the kernel structure during training. 

\subsection{Kernel Embedding Features}

We additionally visualize the features extracted by CHL, AHL, and AHRad on unseen classes (40-49) on AWA2 dataset in zero-shot learning using the t-SNE \cite{maaten2008visualizing} visualization method. We consider two comparative groups: CHL vs. AHL and AHL vs. AHRad. The former comparison illustrates the regularization effect of the de Branges-Rovnyak space, while the latter demonstrates the representational capacity attributed to the multi-kernel learning framework induced by AHRad. In Figure \ref{fig3}, the visual features (point clouds), their corresponding centers (small markers), and the semantic features (large markers) are visualized. Our analysis focuses on the deviation between each class’s visual center and its semantic embedding, which reveals the kernel representation capacity. It can be observed that CHL exhibits the largest overall deviation; AHL reduces this to some extent, while AHRad achieves a substantially lower deviation than both of them.

\section{Conclusion}
This work proposes a family of adaptive hyperbolic kernels based on the curvature-aware de Branges–Rovnyak spaces. By leveraging isometric embeddings between the Poincar\'e ball and these spaces, we effectively reduce the distortion inherent in existing hyperbolic kernels. To further improve adaptability, we incorporate modulation mechanisms that dynamically align the kernels with data geometry. Built upon these foundations, we develop a series of adaptive hyperbolic kernels suitable for different tasks. Extensive experiments on few-shot and zero-shot learning tasks, as well as semantic textual similarity evaluation, demonstrate the superior performance of our proposed method. 

\section{Acknowledgments}
This work was supported by the National Natural Science Foundation of China (No. 62476056, T24B2005, 62306070) and the Social Development Science and Technology Project of Jiangsu Province (No. BE2022811). This work was also supported in part by the Southeast University Start-Up Grant for New Faculty under Grant 4009002309. Furthermore, the work was also supported by the Big Data Computing Center of Southeast University. This work was also supported by “the Fundamental Research Funds for the Central Universities(2242025K30024)”.

\bibliography{aaai2026}

\end{document}